\documentclass[preprint, 3p]{elsarticle}
\makeatletter
\def\ps@pprintTitle{%
 \let\@oddhead\@empty
 \let\@evenhead\@empty
 \def\@oddfoot{\centerline{\thepage}}%
 \let\@evenfoot\@oddfoot}
\makeatother
\usepackage{lineno,hyperref}
\modulolinenumbers[5]
\usepackage{epstopdf}
\usepackage{amsmath,amsfonts,amssymb,graphicx}
\usepackage{hyperref}
\usepackage{color}
\usepackage{xcolor}
\usepackage{mathrsfs}
\usepackage[utf8]{inputenc}
\usepackage[english]{babel} 
\usepackage{tikz}
\usetikzlibrary{arrows}
\usetikzlibrary{positioning,fit,calc}
\tikzset{block/.style={draw,thick,text width=1.5cm,minimum height=1cm,align=center},
line/.style={-latex} }
\usepackage{lipsum,adjustbox}
\usepackage{bookmark}
\usepackage{multirow}

\usepackage{array}
\newcolumntype{M}{>{\centering\arraybackslash}m{1cm}}
\usepackage{hyperref}
\usepackage{color}
\usepackage{xcolor}
\usepackage{mathrsfs}
\usepackage{setspace}
\usepackage[utf8]{inputenc}
\usepackage[english]{babel}
\usepackage{algorithm}
\usepackage{algpseudocode}
\usepackage{bookmark}
\usepackage{subfig}
\usepackage{graphicx}
\usepackage{bm}
\DeclareGraphicsExtensions{.jpg,.pdf,.png,.eps}
\RequirePackage{amsopn}
\RequirePackage{amsfonts}
\RequirePackage{amsthm}

\theoremstyle{definition}
\newtheorem{lem}{Lemma}

\newtheorem{thm}{Theorem}

\begin{document}
\begin{frontmatter}
\title{Recovering Quantized Data with Missing Information Using Bilinear Factorization and Augmented Lagrangian Method}
\author[mymainaddress]{Ashkan~Esmaeili}
\ead{aesmaili@stanford.edu}
\author[mymainaddress]{Kayhan~Behdin}
\ead{behdin\_k@ee.sharif.edu}
\author[mymainaddress]{Sina~Al-E-Mohammad}
\ead{alemohammad\_s@ee.sharif.edu}
\author[mymainaddress]{Farokh~Marvasti\corref{mycorrespondingauthor}}
\ead{fmarvasti@gmail.com}
\cortext[mycorrespondingauthor]{Corresponding author}
\address[mymainaddress]{Advanced Communications Research Institute (ACRI), and\\Electrical Engineering Department, Sharif University of Technology, Tehran, Iran}
\begin{abstract}
\quad In this paper, we propose a novel approach in order to recover a quantized matrix with missing information. We propose a regularized convex cost function composed of a log-likelihood term and a Trace norm term. The Bi-factorization approach and the Augmented Lagrangian Method (ALM) are applied to find the global minimizer of the cost function in order to recover the genuine data. We provide mathematical convergence analysis for our proposed algorithm.
In the Numerical Experiments Section, we show the superiority of our method in accuracy and also its robustness in computational complexity compared to the state-of-the-art literature methods.
\end{abstract}
\end{frontmatter}
\linenumbers
\section{Introduction}\label{Intro}
Matrix Completion (MC) is of interest in many applications and practical settings. In recent years, theoretical advancements and achievements were reached by many authors in MC \cite{candes2009exact,cai2010singular, keshavan2010matrix}. MC has been utilized in a wide range of applications including but not limited to collaborative filtering \cite{koren2009matrix}, sensor networks \cite{biswas2004semidefinite}, prediction, and learning \cite{esmaeili2018transduction}.\par
In this paper, our goal is to recover a matrix of quantized data with missing information (Quantized MC), i.e., some entries of a given matrix are not assigned (NA), and the rest take quantized values rather than continuous values. This problem model has been taken into account by many authors. Quantized MC is found in many practical applications including but not limited to collaborative filtering, learning and content analytics, and sensor network localization \cite{bhaskar}.
\par
A special case of the quantized MC problem is where the quantization is restricted to only $1$ bit, i.e., the observed matrix is a signed version of an oracle continuous-valued matrix. In \cite{davenport20141}, one-bit matrix completion is investigated and the authors propose a convex programming setting to maximize a log-likelihood function to recover the genuine continuous-valued data. In \cite{cai2013max}, a max-norm constrained maximum likelihood estimate is studied. The max-norm is considered as the rank convex relaxation. In \cite{ni2016optimal}, a rank constrained maximum likelihood
estimation is considered, and a greedy algorithm is proposed as an extension of conditional gradient descent, which converges at a linear rate. 
\par
Other works extend the one-bit matrix completion to the general quantized MC such as \cite{lan2014matrix}. In 
\cite{lan2014matrix}, the authors propose the \textbf{Q-MC} method. This method applies a projected gradient based approach in order to minimize a constrained log-likelihood function. The projection step assumes the matrix is included in a convex ball determined by the tuning parameter $\lambda$ (endorsing the Trace norm reduction). \par
In \cite{bhaskar}, a novel method for quantized matrix completion is introduced, where a log-likelihood function is minimized under an exact rank constraint. The rank constraint appears in the assumed dimensions of two factor matrices forming the target matrix. This is a robust method leading to noticeable numerical results in matrix recovery. However, the proposed method in \cite{bhaskar} requires an initial accurate rank estimation, while our proposed method (to be elaborated later) does not require any exact rank estimation at the beginning.
\par
We assume the original data matrix to be low-rank, and consider the trace norm as its convex relaxation. We penalize the log-likelihood function with the trace norm (rank convex surrogate), and leverage the biliniear factorization to solve the regularized log-likelihood function \cite{bilinear}. In our work, the Augmented Lagrangian Method (ALM) is utilized. The alternating gradient descent approach is applied to the two factors to reach the global minimizer of the proposed convex problem.\\
 We have conducted our algorithm in two simulation scenarios (synthetic and Real) and compared our performance in terms of accuracy and computational complexity with other state-of-the-art methods. We will elaborate the analysis in the Simulations Results Section \ref{Discuss}. \par
The rest of the paper is organized as follows: Section \ref{pm} includes the problem model. In Section \ref{proposed}, we propose our algorithms. Next, Section \ref{TA} contains the theoretical convergence analysis. In Section \ref{NE}, we discuss the numerical experiments. In the final Section \ref{conc}, we conclude the paper.
\section{Problem Model}\label{pm}
In this part, we introduce the notations and the mathematical model. Suppose $\boldsymbol{X}_0\in\mathbb{R}^{m\times n}$ is a low rank matrix which we have access to quantized observations for some of its entries. Our goal is to recover unobserved entries of $\boldsymbol{X}_0$ as well as enhancing observed ones. Formally, we define the set of observed entries as:
$$\Omega=\{(i,j):(\boldsymbol{X}_0)_{ij} \,\text{has been observed}\}$$,
where $\boldsymbol{X}_{ij}$ denotes the $(i,j)$th entry of matrix $\boldsymbol{X}$. In addition, $\boldsymbol{Y}\in\mathbb{R}^{m\times n}$ denotes the quantized measurements matrix with zeros in locations that have not been observed. For $(i,j)\in\Omega$, we have
$$\boldsymbol{Y}_{ij}=Q((\boldsymbol{X}_0)_{ij}+ \boldsymbol{\epsilon}_{ij})$$,
where $Q(.)$ is the quantization rule based on the quantization levels and $\boldsymbol{\epsilon}$ is the noise matrix with i.i.d entries which have been drawn from the logistic distribution with cumulative distribution function of $\Phi(x)$ \cite{bhaskar}, where
\begin{equation}
\Phi(x)=\frac{1}{1+\exp(-x)}.
\end{equation}
To denoise the observed entries, we exploit the maximum likelihood estimation for the logistic distribution. Let $f_{ij}(x)$ denote the following likelihood function on $x_{ij}$. We have 
\begin{equation}
f_{ij}(x)=\Phi(U_{ij}-x)-\Phi(L_{ij}-x)
\end{equation}
Here, $U_{ij},L_{ij}$ are the upper and lower quantization bounds of the $ij$-th entry of the quantized matrix $\boldsymbol{Y}$. Multiplying likelihood functions for the observed entries, our log-likelihood function is defined as follows:
\begin{equation}\label{loglike}
F_{\boldsymbol{Y}}(\boldsymbol{X})=\sum_{(i,j) \in \Omega}\log f_{ij}(\boldsymbol{X}_{ij}),
\end{equation}
which is a function of the observation matrix $\boldsymbol{Y}$ and the test matrix $\boldsymbol{X}$. This log-likelihood function is, however, independent of unobserved entries, therefore, we exploit the low-rank structure of the initial matrix to recover them. Here, our assumption is that the initial matrix $\boldsymbol{X}_0$, from which the observation matrix $\boldsymbol{Y}$ has been produced, is low rank. Hence, $r^*=\text{rank}(\boldsymbol{X}_0)$ is small compared to $m$ and $n$. Therefore, we can write the desired optimization problem as a log-likelihood maximization under a low-rank constraint as follows:
\begin{align}\label{p4}
\nonumber \underset{\boldsymbol{X}}{\min} ~ -F_{\boldsymbol{Y}}(\boldsymbol{X})\\
s.t. ~  \text{rank}(\boldsymbol{X})\leq r^*
\end{align} \\
Since the rank constraint is not convex and generally difficult to handle, our purpose in the rest of this section is to provide a tractable relaxation of this problem. To do this, unlike \cite{bhaskar} where the authors are restricted to have an exact rank estimation, we aim to consider the penalized problem. This leads us to use the convex surrogate of the rank function (trace norm), denoted as $\|.\|_*$, and as a result, we take advantage of the convexity of the regularized cost function. We first modify the problem in \ref{p4} as follows:
\begin{align}\label{p5}
\nonumber \underset{\boldsymbol{X}}{\min} ~ -F_{\boldsymbol{Y}}(\boldsymbol{X})\\
s.t. ~  \|\boldsymbol{X}\|_*\leq k.
\end{align} 
Although this modified problem is convex, we still tend to have a problem without constaints to solve by generic convex optimization methods. Therefore, we regularize this constrained problem by adding the constraint as a penalty term to the main cost function. Hence, the following unconstrained problem is obtained: 
\begin{equation} \label{p6}
\underset{\boldsymbol{X}}{\min}\quad G(\boldsymbol{X},\lambda),
\end{equation}
where $G(\boldsymbol{X},\lambda):= -F_{\boldsymbol{Y}}(\boldsymbol{X})+\lambda \|\boldsymbol{X}\|_*$. The cost function here consists of two terms. The first term assures the solution is consistent with the quantized observations w.r.t the log-likelihood criterion. The second term controls how low rank the result is by trying to reduce the trace norm. The parameter $\lambda$ determines how effective each term in the cost function is. For example, if $\lambda$ is increased, we expect to obtain a low rank matrix compared to when $\lambda$ is smaller. We will later prove this result formally, but for now, we just rely on this insight and assume $\lambda$ is chosen properly. Thus, the solution satisfies the constraint $\|X\|_*\leq k$.   \\
\section{The Proposed Algorithm}\label{proposed}
In this section, we propose a novel algorithm to solve problem (\ref{p6}). Contrary to \cite{bhaskar,lan2014matrix} which apply the gradient descent method (and projection to handle constraints) directly to their proposed problems, we modify our cost function to obtain a differentiable form which will be easier to handle. To this end, we utilize the idea of bilinear factorization for trace norm minimization as introduced in \cite{bilinear}. In this approach, we decompose $\boldsymbol{X}$ as $\boldsymbol{X}=\boldsymbol{U}\boldsymbol{V}^T$ where $\boldsymbol{U}\in\mathbb{R}^{m\times r}$, $\boldsymbol{V}\in\mathbb{R}^{n\times r}$, and $r\leq\min(m,n)$ is an estimation of the rank of the solution. Therefore, our problem can be reformulated as \cite{bilinear}
\begin{align}\label{main}
\nonumber & \underset{\boldsymbol{X},\boldsymbol{U},\boldsymbol{V}}{\min} ~ -F_{\boldsymbol{Y}}(\boldsymbol{X})+\frac{\lambda}{2}(\|\boldsymbol{U}\|_F^2+\|\boldsymbol{V}\|_F^2)\\
& s.t. ~  \boldsymbol{X}=\boldsymbol{U}\boldsymbol{V}^T.
\end{align} 
We leverage the ALM approach as in \cite{bilinear} to solve problem (\ref{main}). By analyzing subproblems brought up by ALM and substituting closed form solutions for subproblems with such solutions, the overall procedure can be described as in Algorithm \ref{Algorithm 1}. In this algorithm, $\mathcal{U}$ denotes the uniform distribution and matrices $\boldsymbol{U}$ and $\boldsymbol{V}$ are initialized randomly because no better estimation of their final value can be obtained easily. Note that the subproblem 
$$\min_{\boldsymbol{X}}-F_{\boldsymbol{Y}}(\boldsymbol{X})+\frac{\rho}{2}\|\boldsymbol{X}-(\boldsymbol{U}\boldsymbol{V}^T-\frac{\boldsymbol{\Lambda}}{\rho})\|_F^2$$
is convex and differentiable and can be solved by many different methods, like the gradient method which we use in our simulations. 
\begin{algorithm}[h!] 
	\small
	\caption{}\label{Algorithm 1}
	\begin{algorithmic}[1]
		\State Input:
		\State Observation matrix $\boldsymbol{Y}$ 
		\State The set of observed indices $\Omega$
		\State The quantization lower and upper bounds $U_{ij}, L_{ij}$
		\State The regularization factor $\lambda$
		\State The ALM penalty $\rho$
		\State Initial estimation of the rank $r$
		\State output:
		\State The completed matrix $\boldsymbol{X}^*$
		\Procedure {QMC-BIF}{$\boldsymbol{Y}, \Omega, U_{ij}, L_{ij}, \lambda, ,r,\rho$}
		\State $\boldsymbol{U} \sim\mathcal{U}([0,1]^{m\times r})$
		\State $\boldsymbol{V} \sim \mathcal{U}([0,1]^{m\times r})$
		\State $\boldsymbol{Z} \leftarrow \boldsymbol{Y}$
		\State $ \boldsymbol{\Lambda} = \boldsymbol{0}_{mn}$
		\While {not converged}
		\While {not converged}
		\State $\boldsymbol{U} \leftarrow (\rho \boldsymbol{Z}+\boldsymbol{\Lambda})\boldsymbol{V}(\rho \boldsymbol{V}^T\boldsymbol{V}+\lambda \boldsymbol{I}_r)^{-1}$
		\State $\boldsymbol{V} \leftarrow (\rho \boldsymbol{Z}+\boldsymbol{\Lambda})^T\boldsymbol{U}(\rho \boldsymbol{U}^T\boldsymbol{U}+\lambda \boldsymbol{I}_r)^{-1}$
		\EndWhile
		\State $\boldsymbol{\Lambda} \leftarrow \boldsymbol{\Lambda} + \rho(\boldsymbol{Z}-\boldsymbol{UV}^T)$
		\State $\boldsymbol{Z} \leftarrow \textrm{argmin}_{\boldsymbol{X}}{-F_{\boldsymbol{Y}}(\boldsymbol{X})}+\frac{\rho}{2}\|\boldsymbol{X}-(\boldsymbol{UV}^T-\frac{\boldsymbol{\Lambda}}{\rho})\|_F^2$
		\EndWhile 
		\State
		\Return $\boldsymbol{X}^* \leftarrow \boldsymbol{Z}$
		\EndProcedure	
	\end{algorithmic} 
\end{algorithm} 
The advantage of our algorithm to the one in \cite{bhaskar} is that Theorem 1 in \cite{bilinear} states that any rank estimation satisfying $r\geq r^*$ leads to solutions that are also solutions of (\ref{p6}), removing the need for an accurate rank estimation. It is needless to say that no projection is needed in our method in contrast to \cite{bhaskar,lan2014matrix}.
\section{Theoretical Analysis}\label{TA}
The problems represented in (\ref{p5}) and (\ref{p6}) are obviously different. There is a constraint in problem (\ref{p5}) guaranteeing that $\|\boldsymbol{X}\|_*\leq k$. In (\ref{p6}) however, as we discussed earlier, the parameter $\lambda$ controls how low-rank the minimizer is. In this section, we analyze the effect of replacing the constraint with a penalty term in our problem and also the appropriate value for $\lambda$. First, we add an assumptions to our problem.
\\
\textbf{A}: $\boldsymbol{X}_0$ is a solution of problem (\ref{p5}). \\
This assumption assures the observations are enough to recover the initial matrix because if this assumption is violated, the ideal maximum likelihood estimation can not recover the information lost, therefore, we can not guarantee the data is recoverable by our method which is a penalized version of the maximum likelihood estimation. Based on this explanation, assumption (A) is reasonable. \\
  Let
$$S^*(\lambda)=\arg\underset{\boldsymbol{X}}{\min}~ G(\boldsymbol{X},\lambda) =\arg\underset{\boldsymbol{X}}{\min} -F_{\boldsymbol{Y}}(\boldsymbol{X})+\lambda \|\boldsymbol{X}\|_*.$$
The set $S^*(\lambda)$ contains all global solutions of the problem in (\ref{p6}) for a fixed value of $\lambda$. It is easy to observe this set is bounded, i.e. 
$r(\lambda):=\sup_{\boldsymbol{X}\in S^*(\lambda)}\|\boldsymbol{X}\|_*<\infty$.
Our expectation of the effect of $\lambda$ can be formally stated as $r(\lambda)$ is a decreasing function of $\lambda$. This result is proven in the following lemma.
\begin{lem}\label{lemma1}
	The function $r(\lambda)=\sup_{\boldsymbol{X}\in S^*(\lambda)}\|\boldsymbol{X}\|_*$ is decreasing over $\lambda$.
\end{lem}
\begin{proof}
Assume $\lambda_1>\lambda_2$. In addition, assume $\boldsymbol{X}^*(\lambda_1)$ and $\boldsymbol{X}^*(\lambda_2)$ are chosen arbitrarily from $S^*(\lambda_1)$ and $S^*(\lambda_2)$, respectively. Using the minimizer definition for the problem in \ref{p6}, we have:  
\begin{equation}\label{eq6}
\lambda_1\|\boldsymbol{X}^*(\lambda_1)\|_*-F_{\boldsymbol{Y}}(\boldsymbol{X}^*(\lambda_1))\leq \lambda_1\|\boldsymbol{X}^*(\lambda_2)\|_*-F_{\boldsymbol{Y}}(\boldsymbol{X}^*(\lambda_2))
\end{equation}
\begin{equation}\label{eq7}
\lambda_2\|\boldsymbol{X}^*(\lambda_2)\|_*-F_{\boldsymbol{Y}}(\boldsymbol{X}^*(\lambda_2))\leq \lambda_2\|\boldsymbol{X}^*(\lambda_1)\|_*-F_{\boldsymbol{Y}}(\boldsymbol{X}^*(\lambda_1))
\end{equation}
By adding (\ref{eq6}) and (\ref{eq7}), the following inequality immediately follows:
\begin{align}
\lambda_1\|\boldsymbol{X}^*(\lambda_1)\|_*+\lambda_2\|\boldsymbol{X}^*(\lambda_2)\|_*\leq \lambda_1\|\boldsymbol{X}^*(\lambda_2)\|_*+\lambda_2\|\boldsymbol{X}^*(\lambda_1)\|_*
\end{align}
\begin{align}
 \Rightarrow(\lambda_1-\lambda_2)\|\boldsymbol{X}^*(\lambda_1)\|_*\leq(\lambda_1-\lambda_2)
 \|\boldsymbol{X}^*(\lambda_2)\|_*
\end{align}
Since $\lambda_1-\lambda_2>0$, $\|\boldsymbol{X}^*(\lambda_1)\|_*\leq \|\boldsymbol{X}^*(\lambda_2)\|_*$ and since this inequality holds for every two arbitrary matrices from $S^*(\lambda_1)$ and $S^*(\lambda_2)$, the desired result is deduced. 
\end{proof}
Lemma \ref{lemma1} proves our insight about $\lambda$. However, it does not provide any suggestion about the appropriate value of $\lambda$. Before we proceed, we prove a lemma to characterize $S^*(\lambda)$ which will be used later.
\begin{lem}\label{equality}
	For $\boldsymbol{X}\in S^*(\lambda)$ for a fixed $\lambda$, the value $\boldsymbol{X}_{ij}$ is equal to some constant $c_{ij}$ for $(i,j)\in\Omega$. 
\end{lem}
\begin{proof}
	If $S^*(\lambda)$ is a singleton, the result is trivial. Therefore, assume there exists $\boldsymbol{X}_1$ and $\boldsymbol{X}_2$ in $S^*(\lambda)$ such that $\boldsymbol{X}_1\neq\boldsymbol{X}_2$. It is easy to observe $S^*(\lambda)$ is bounded and each observed entry of $\boldsymbol{X}$ can be bounded as $l_{ij}\leq \boldsymbol{X}_{ij}\leq u_{ij}$, $(i,j)\in\Omega$, for $\boldsymbol{X}\in S^*(\lambda)$. Therefore, there exist positive constants $d_{ij}>0$ such that  $\frac{d^2}{dx^2}(-\log f_{ij}(x))>d_{ij}$ for $l_{ij}\leq x\leq u_{ij}$. As a result, $f_{ij}$ is strongly (and strictly) convex in $S^*(\lambda)$ for $(i,j)\in\Omega$ or equivalently,
	$$-\log f_{ij}(\alpha x + (1-\alpha)y)\leq -\alpha \log f_{ij}(x) -(1-\alpha) \log f_{ij}(y)-\frac{1}{2}d_{ij}\alpha(1-\alpha)(x-y)^2$$
	for $l_{ij}\leq x,y\leq u_{ij}$ and $0\leq\alpha\leq 1$.
	\\
	Define $\boldsymbol{X}_\alpha=\alpha \boldsymbol{X}_1+(1-\alpha)\boldsymbol{X}_2$ for $\alpha\in(0,1)$. Since problem (\ref{p6}) is convex, $S^*(\lambda)$ is a convex set and $\boldsymbol{X}_{\alpha}\in S^*(\lambda)$. Therefore, one can write
	\begin{align*}
	-F_{\boldsymbol{Y}}(\boldsymbol{X}_\alpha)& =\sum_{(i,j)\in\Omega} -\log f_{ij}((\boldsymbol{X}_\alpha)_{ij})=\sum_{(i,j)\in\Omega} -\log f_{ij}(\alpha(\boldsymbol{X}_1)_{ij}+(1-\alpha)(\boldsymbol{X}_2)_{ij})\leq\\
	&\sum_{(i,j)\in\Omega} [-\alpha\log f_{ij} ((\boldsymbol{X}_1)_{ij})   -(1-\alpha)\log f_{ij} ((\boldsymbol{X}_2)_{ij})  -\frac{1}{2}d_{ij}\alpha(1-\alpha)((\boldsymbol{X}_1)_{ij}-(\boldsymbol{X}_2)_{ij})^2]\leq \\
	& -\alpha F_{\boldsymbol{Y}}(\boldsymbol{X}_1)-(1-\alpha) F_{\boldsymbol{Y}}(\boldsymbol{X}_2)-\frac{1}{2}d\alpha(1-\alpha)\sum_{(i,j)\in\Omega} ((\boldsymbol{X}_1)_{ij}-(\boldsymbol{X}_2)_{ij})^2
	\end{align*}
	where $d=\min d_{ij}>0$. We know that trace norm is convex and therefore, $\|\boldsymbol{X}_{\alpha}\|_*\leq \alpha\|\boldsymbol{X}_1\|_*+(1-\alpha)\|\boldsymbol{X}_2\|_*$. Overall, we have 
	\begin{align*}
	G(\boldsymbol{X}_{\alpha},\lambda)&=-F_{\boldsymbol{Y}}(\boldsymbol{X}_{\alpha})+\lambda\|\boldsymbol{X}_{\alpha}\|_* \leq \\
	& \alpha [-F_{\boldsymbol{Y}}(\boldsymbol{X}_1)+\lambda\|\boldsymbol{X}_1\|_*] +(1-\alpha)[-F_{\boldsymbol{Y}}(\boldsymbol{X}_2)+\lambda\|\boldsymbol{X}_2\|_*] -\frac{1}{2}d\alpha(1-\alpha)\sum_{(i,j)\in\Omega} ((\boldsymbol{X}_1)_{ij}-(\boldsymbol{X}_2)_{ij})^2=\\
	& \alpha G(\boldsymbol{X}_1,\lambda) + (1-\alpha)G(\boldsymbol{X}_2,\lambda)-\frac{1}{2}d\alpha(1-\alpha)\sum_{(i,j)\in\Omega} ((\boldsymbol{X}_1)_{ij}-(\boldsymbol{X}_2)_{ij})^2= \\
	& G(\boldsymbol{X}_1,\lambda)-\frac{1}{2}d\alpha(1-\alpha)\sum_{(i,j)\in\Omega} ((\boldsymbol{X}_1)_{ij}-(\boldsymbol{X}_2)_{ij})^2.
	\end{align*}
	The last inequality is true since all minima of problem (\ref{p6}) are global minimums (a result of convexity), so $G(\boldsymbol{X},\lambda)$ is constant for $\boldsymbol{X}\in S^*(\lambda)$. One can write
	$$G(\boldsymbol{X}_1,\lambda)=G(\boldsymbol{X}_{\alpha},\lambda)\leq G(\boldsymbol{X}_1,\lambda)-\frac{1}{2}d\alpha(1-\alpha)\sum_{(i,j)\in\Omega} ((\boldsymbol{X}_1)_{ij}-(\boldsymbol{X}_2)_{ij})^2$$
	and because this inequality is true for $0\leq\alpha\leq 1$, we must have 
	$$\sum_{(i,j)\in\Omega} ((\boldsymbol{X}_1)_{ij}-(\boldsymbol{X}_2)_{ij})^2=0$$
	or equivalently, 
	$$(\boldsymbol{X}_1)_{ij}=(\boldsymbol{X}_2)_{ij}\quad \forall (i,j)\in\Omega.$$
\end{proof}
The following theorem, lets us know the appropriate range of values for $\lambda$.
\begin{thm}
If $\lambda\geq \lambda^*=|\Omega|$  then, $r(\lambda)\leq k$ under assumption (A).
\end{thm}
\begin{proof}
Assume $\lambda\geq \lambda^*$, and $\boldsymbol{X}^*(\lambda)\in S^*(\lambda)$ such that $\|\boldsymbol{X}^*(\lambda)\|_*=k+\epsilon$. We prove that this will be impossible by contradiction. Let $\boldsymbol{X}^*(\lambda)=\boldsymbol{U}\text{diag}(\sigma_1,\sigma_2,...,\sigma_q)\boldsymbol{V}^T$ be the SVD decomposition of $\boldsymbol{X}^*(\lambda)$, where $q=\min(m,n)$ and the singular values are sorted as $\sigma_1\geq ... \geq \sigma_q$. We choose the values $\epsilon_1, .. , \epsilon_k$ such that first, $0\leq \epsilon_i \leq \sigma_i, \forall 1\leq i \leq k$ and second, $\sum_{i=1}^{q}\epsilon_i=\epsilon$. It is worth noting that this is possible since $\|\boldsymbol{X}^*(\lambda)\|_*=\sum_{i=1}^{q}\sigma_i=k+\epsilon>\epsilon$. We define $\boldsymbol{Z}$ as follows:
\begin{equation}
\boldsymbol{Z}=\boldsymbol{U}\text{diag}(\sigma_1-\epsilon_1,...,\sigma_q-\epsilon_q)\boldsymbol{V}^T.
\end{equation}
Hence, $\|\boldsymbol{Z}\|_*=\sum_{i=1}^{q}\sigma_i-\sum_{i=1}^{q}\epsilon_i=k$.\\

The matrix $\boldsymbol{Z}$ holds in the constraint of the problem in (\ref{p5}). By using the minimizer property and our assumptions, it follows that:
$F_{\boldsymbol{Y}}(\boldsymbol{Z}) \leq F_{\boldsymbol{Y}}(\boldsymbol{X}_0)$. In addition, one can write
\begin{align}
\|\boldsymbol{Z}-\boldsymbol{X}^*(\lambda)\|_F^2  = &
\|\boldsymbol{U}\text{diag}(\sigma_1,...\sigma_q)\boldsymbol{V}^T-\boldsymbol{U}\text{diag}(\sigma_1-\epsilon_1,...,\sigma_q-\epsilon_q)\boldsymbol{V}^T||_F^2=\\
& \|\boldsymbol{U}\text{diag}(\epsilon_1,...,\epsilon_q)\boldsymbol{V}^T\|
_F^2=\text{tr}(\boldsymbol{USV}^T\boldsymbol{VS}^T\boldsymbol{U}^T)=\text{tr}(\boldsymbol{USS}^T\boldsymbol{U}^T)= \\ & \text{tr}(\boldsymbol{U}^T\boldsymbol{USS}^T)=\text{tr}(\boldsymbol{SS}^T)=\sum_{i=1}^{q}\epsilon_i^2=\epsilon^2
\end{align}
We have shown that $\|\boldsymbol{Z}-\boldsymbol{X}^*(\lambda)\|_F\leq \epsilon$. Thus, for each $(i,j)$ we have $|\boldsymbol{Z}_{ij}-\boldsymbol{X}^*_{ij}(\lambda)|\leq\epsilon$.
It is worth noting that $\log(f_{ij}(.))$ is Lipschitz continuous with coefficient $1$. Thus, 
\begin{equation}
|\log f_{ij}(\boldsymbol{Z}_{ij})-\log f_{ij}(\boldsymbol{X}^*_{ij}(\lambda))| \leq |\boldsymbol{Z}_{ij}-\boldsymbol{X}^*_{ij}(\lambda)|\leq \epsilon
\end{equation}
Using the Triangle inequality, we have:
\begin{equation}\label{lips}
F_{\boldsymbol{Y}}(\boldsymbol{X}^*(\lambda))-F_{\boldsymbol{Y}}(\boldsymbol{Z})\leq \sum_{(i,j)\in \Omega}|\log f_{ij}(\boldsymbol{Z}_{ij})-\log f_{ij}(\boldsymbol{X}_{ij}^*(\lambda))|\leq |\Omega| \epsilon
\end{equation}
Using the minimizer concept for problem (\ref{p6}), we have:
\begin{equation}\label{eq15}
\lambda\|\boldsymbol{X}_0\|_*-F_{\boldsymbol{Y}}(\boldsymbol{X}_0)-[\lambda\|\boldsymbol{X}^*(\lambda)\|_*-F_{\boldsymbol{Y}}(\boldsymbol{X}^*(\lambda))]\geq0
\end{equation}
On the other hand,
\begin{align}\label{eqcontr}
\nonumber \lambda\|\boldsymbol{X}_0\|_*-F_{\boldsymbol{Y}}(\boldsymbol{X}_0)-[\lambda\|\boldsymbol{X}^*(\lambda)\|_*-F_{\boldsymbol{Y}}(\boldsymbol{X}^*(\lambda))]\leq\lambda k-\lambda(k+\epsilon)+F_{\boldsymbol{Y}}(\boldsymbol{X}^*(\lambda))-F_{\boldsymbol{Y}}(\boldsymbol{X}_0)\\
\nonumber=-\epsilon\lambda+F_{\boldsymbol{Y}}(\boldsymbol{X}^*(\lambda))-F_{\boldsymbol{Y}}(\boldsymbol{X}_0)\\
\nonumber \leq -\epsilon\lambda+F_{\boldsymbol{Y}}(\boldsymbol{X}^*(\lambda))-F_{\boldsymbol{Y}}(\boldsymbol{Z})\\
\leq \epsilon(|\Omega|-\lambda)=\epsilon(\lambda^*-\lambda)\leq0
\end{align}
The first inequality is true since $\boldsymbol{X}_0$ is a feasible point of the problem \ref{p5}, i.e., $\|\boldsymbol{X}_0\|_*\leq k$.
The second inequality follows from the minimizer property definition (problem \ref{p5}), and the third inequality follows from (\ref{lips}). This inequality proves $\boldsymbol{X}_0\in S^*(\lambda)$ because $G(\boldsymbol{X}_0,\lambda)\leq G(\boldsymbol{X}^*(\lambda),\lambda)$. However, this is a contradiction because lemma \ref{equality} implies $F_{\boldsymbol{Y}}(\boldsymbol{X}_0)=F_{\boldsymbol{Y}}(\boldsymbol{X}^*(\lambda))$ while $\|\boldsymbol{X}_0\|_*=k<k+\epsilon=\|\boldsymbol{X}^*(\lambda)\|_*$, resulting in $G(\boldsymbol{X}_0,\lambda)<G(\boldsymbol{X}^*(\lambda),\lambda)$.
\end{proof}

\section{Numerical Experiments}\label{NE}
In this section, we provide numerical experiments for two scenarios:
\begin{itemize}
\item 
Synthetic Dataset
\item 
Real Dataset
\end{itemize}
We also compare the performance of our proposed approach to two state-of-the-art methods as explained hereunder:
\begin{itemize}
\item
\textbf{SPARFA-Lite}
In \cite{lan2014quantized}, the authors propose a quantized matrix completion method for personalized learning. This method is considered as a robust method in quantized matrix completion and the details of the algorithm have been explained earlier in our Introduction \ref{Intro}. This method is similar to the version \textbf{Q-MC} presented in \cite{lan2014matrix}.
They apply a projected gradient based approach in order to minimize the constrained log-likelihood function introduced in \ref{loglike}. The projection step assumes the matrix is included in a convex ball determined by the tuning parameter $\lambda$ (endorsing trace norm reduction).  
	\item \textbf{Approximate Projected Gradient Method} (\cite{bhaskar})
	This method is proposed in \cite{bhaskar}. It is worth noting that we do not confine our problem model to any hard rank constraints. As explained in Section \ref{proposed}, the matrix size we take 
\end{itemize}

\subsection{Synthetic Dataset}
We have generated two orthonormal matrices and a diagonal matrix containing singular values (with uniform distribution on $[0,1]$) determining the rank of the desired matrix. Now, multiplying the three components of the SVD, the initial synthetic matrix is resulted. Since SPARFA-Lite algorithm considers integer values, we also desire to compare the results in the similar platform. Thus, we normalize the initial synthesized matrix onto the interval $[1,n]$, where $n$ denotes the number of quantization levels. This normalization preserves the low-rank property (the rank may be only increased by $1$ which is ignorable in our settings). Next, we map the values of this matrix entries onto integer bounds. We apply a random mask to deliberately induce missing pattern on the data. In discussion subsection \ref{Discuss}, we elaborate the results achieved as well as the missing percentage, rank values, and the dimensions selected for the synthetic data. 
\subsection{MovieLens (100k) Dataset}
In \cite{harper2016movielens} the Movielens dataset is provided and clearly introduced. The detailed explanations regarding this dataset can be found in 
\cite{harper2016movielens} and \cite{Movielens}. Here we briefly discuss this dataset properties.
It contains $100,000$ ratings (instances) $(1-5)$ from $943$ users on $1682$ movies, where each user has rated at least $20$ movies. Our purpose is to predict the ratings which have not been recorded or completed by users.
We assume this rating matrix is a quantized version of an oracle genuine low-rank matrix and recover it using our algorithm. Then, a final quantization can be applied to predict the missing ratings.\\
\subsection{Discussion on Simulation Results}\label{Discuss}
The dimensions of the synthesized matrix are chosen to be $250\times350$. Two missing percentages of $10\%$ and $15\%$ are considered for our simulations. The number of quantization levels are also assumed to be equal to $10$ and $15$, respectively.
In Figure \ref{f1} and Figure \ref{f2}, we observe the relative error (accuracy) in recovery and also the computational time in seconds measured on an @Intel Core i7 6700 HQ 16 GB RAM system using MATLAB \textsuperscript{\textregistered}. It is worth noting that our method outperforms the two state-of the-art methods in terms of relative error in recovery as depicted in Figure \ref{f1}. Figure \ref{f2}, shows that the computational complexity of our method is comparable to those of the state-of-the-art methods except that for the increased rank, the computational time required by our method increases with a higher slope in comparison to the two others. 
\begin{figure*}
\centering
	\subfloat[]{\includegraphics[width=0.5\linewidth]{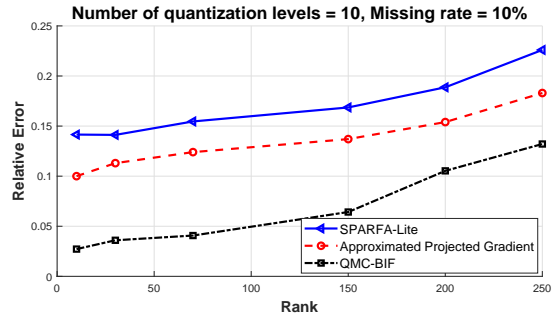}}\quad
	\subfloat[]{\includegraphics[width=0.5\linewidth]{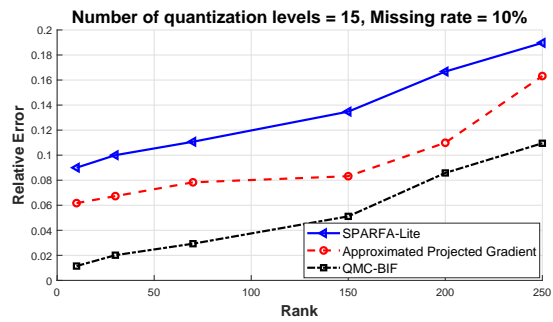}}\\
	\subfloat[]{\includegraphics[width=0.5\linewidth]{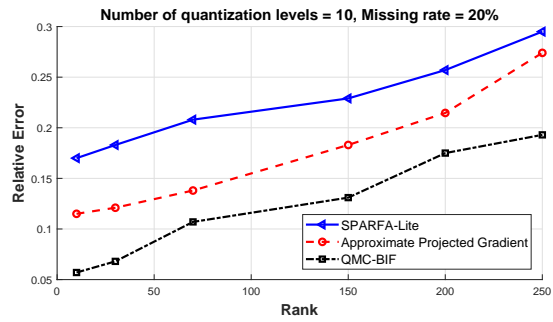}}\quad
	\subfloat[]{\includegraphics[width=0.5\linewidth]{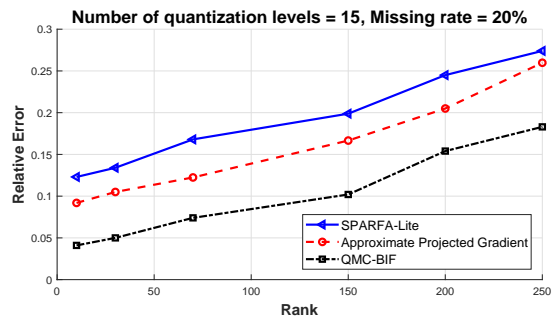}}
	\caption{Relative error values for different methods, quantization levels (QMC-BIF is our proposed method), and missing rates.}\label{f1}
\end{figure*}
\begin{figure*}
\centering
	\subfloat[]{\includegraphics[width=0.5\linewidth]{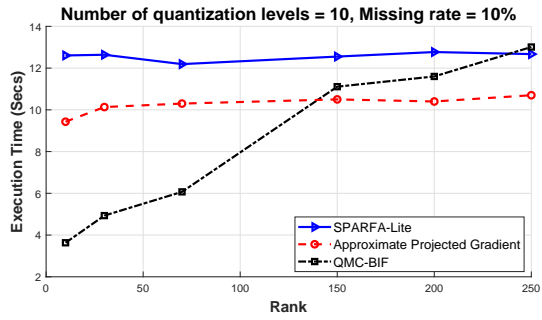}}\quad
	\subfloat[]{\includegraphics[width=0.5\linewidth]{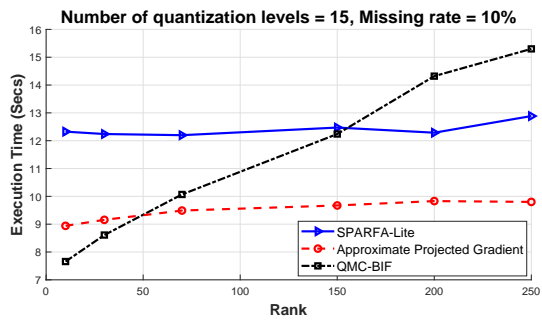}}\\
	\subfloat[]{\includegraphics[width=0.5\linewidth]{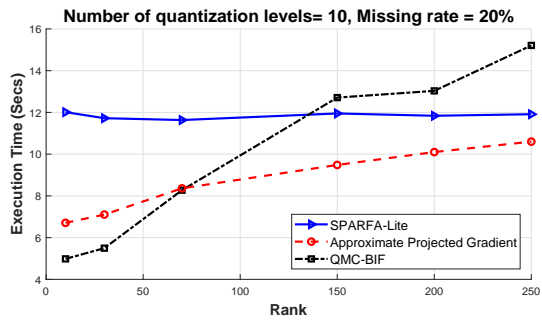}}\quad
	\subfloat[]{\includegraphics[width=0.5\linewidth]{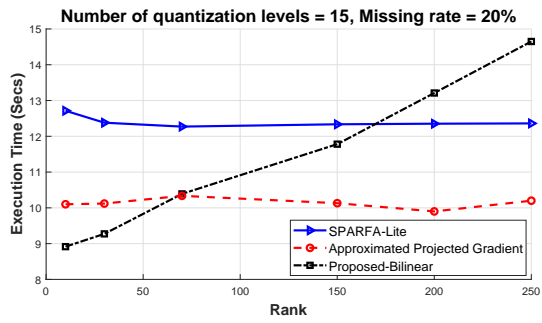}}
	\caption{Runtime values for different methods (QMC-BIF is our proposed method), quantization levels, and missing rates.}\label{f2}
\end{figure*}
\begin{table}\label{T1}
\caption{RMSE Values of Different Methods on MovieLens100K Dataset}
\begin{center}
\begin{tabular}{ |c|c|c|c| }
\hline
 Missing & QMC-BIF(Proposed) & Approximate Projected Gradient & SPARFA-Lite\\ 
 \hline
 $10\%$ & $0.943$ & $1.180$ & $1.316$\\ 
 \hline
 $20\%$ &  $1.375$ & $1.703$ & $1.825$\\ 
 \hline
\end{tabular}
\end{center}
\end{table}
\begin{table}\label{T2}
\caption{Computational Runtime of Different Methods on MovieLens100K Dataset in Secs}
\begin{center}
\begin{tabular}{ |c|c|c|c| }
\hline
 Missing & QMC-BIF(Proposed) & Approximate Projected Gradient   & SPARFA-Lite\\ 
 \hline
 $10\%$ & $612$ & $573$ & $635$\\ 
 \hline
 $20\%$ &  $726$ & $582$ & $647$\\ 
 \hline
\end{tabular}
\end{center}
\end{table}
The Tables \ref{T1} and \ref{T2} illustrates the runtime and RMSE values of the thrree methods (including ours) on the \textit{MovieLens100K} dataset:
Again we observe the outperformance of our accuracy compared to the two other methods. However, the runtime of our method increase more significantly when the rank is increased or more observations are available.
\section{Conclusion}\label{conc}
In this paper, we propose a novel approach in quantized matrix completion using bilinear factorization and ALM method to minimize a penalized log-likelihood convex function. We have established theoretical convergence guarantees for our approach and also claim that our proposed method is not sensitive to an exact initial rank estimation. The numerical experiments also illustrate the superior accuracy of our method compared to state-of-the-art methods as well as reasonable computational complexity.
\section*{References}
\bibliographystyle{elsarticle-num} 
\bibliography{refQMC.bib}
\end{document}